\documentclass{article}

     \usepackage[final, nonatbib]{neurips_2020}
\usepackage[numbers]{natbib}
\usepackage[ruled,vlined]{algorithm2e}

\usepackage{float}
\usepackage{algorithmic}
\usepackage{todonotes}
\usepackage{amsmath, amssymb, amsfonts, amsthm}
\usepackage[utf8]{inputenc} 
\usepackage[T1]{fontenc}    
\usepackage{hyperref}       
\usepackage{url}            
\usepackage{booktabs}       
\usepackage{amsfonts}       
\usepackage{nicefrac}       
\usepackage{microtype}      
\usepackage{graphicx}
\usepackage{cleveref}
\usepackage{paralist}
\usepackage{bbm}
\usepackage{wrapfig}
\usepackage{enumitem}
\setlist{leftmargin=4.5mm}
\usepackage[ruled,vlined]{algorithm2e}
\newcommand{\ssvsp}{\vspace{-0.2cm}}
\newcommand{\svsp}{\vspace{-0.2cm}}
\newcommand{\data}{\mathcal{D}}
\newcommand{\trdata}{\mathcal{D}}

\newcommand{\model}{\mathcal{M}}

\newcommand{\KL}{\text{KL}}

\DeclareMathOperator*{\argmax}{arg \ max}
\DeclareMathOperator*{\argmin}{arg \ min}

\usepackage{thm-restate}
\usepackage{selectp}

\theoremstyle{definition}

\newtheorem{theorem}{Theorem}[section]

\newtheorem{lemma}[theorem]{Lemma}

\title{A Bayesian Perspective on Training Speed \\ and Model Selection}

\author{%
  Clare Lyle \thanks{OATML Group, University of Oxford. Correspondence to \texttt{clare.lyle@cs.ox.ac.uk}} \\
   \And
   Lisa Schut$^\dag$ \\
   \AND
   Binxin Ru$^\dag$ \\
   \And
   Yarin Gal$^\dag$ \\
   \And
   Mark van der Wilk\thanks{Imperial College London} \\
}

\begin{document}

\maketitle

\begin{abstract}
We take a Bayesian perspective to illustrate a connection between training speed and the marginal likelihood in linear models. This provides two major insights: first, that a measure of a model's training speed can be used to estimate its marginal likelihood. Second, that this measure, under certain conditions, predicts the relative weighting of models in linear model combinations trained to minimize a regression loss. We verify our results in model selection tasks for linear models and for the infinite-width limit of deep neural networks. We further provide encouraging empirical evidence that the intuition developed in these settings also holds for deep neural networks trained with stochastic gradient descent. Our results suggest a promising new direction towards explaining why neural networks trained with stochastic gradient descent are biased towards functions that generalize well. 

\end{abstract}

\section{Introduction}
Choosing the right inductive bias for a machine learning model, such as convolutional structure for an image dataset, is critical for good generalization. The problem of \emph{model selection} concerns itself with identifying good inductive biases for a given dataset.
In Bayesian inference, the marginal likelihood (ML) provides a principled tool for model selection. In contrast to cross-validation, for which computing gradients is cumbersome, the ML can be conveniently maximised using gradients when its computation is feasible.
Unfortunately, computing the marginal likelihood for complex models such as neural networks is typically intractable. Workarounds such as variational inference suffer from expensive optimization of many parameters in the variational distribution and differ significantly from standard training methods for Deep Neural Networks (DNNs), which optimize a single parameter sample from initialization. A method for estimating the ML that closely follows standard optimization schemes would pave the way for new practical model selection procedures, yet remains an open problem.

A separate line of work aims to perform model selection by predicting a model's test set performance. This has led to theoretical and empirical results connecting training speed and generalization error \citep{hardt2015train, jiang2019fantastic}. This connection has yet to be fully explained, as most generalization bounds in the literature depend only on the final weights obtained by optimization, rather than on the trajectory taken during training, and therefore are unable to capture this relationship. Understanding the link between training speed, optimization and generalization thus presents a promising step towards developing a theory of generalization which can explain the empirical performance of neural networks.

In this work, we show that the above two lines of inquiry are in fact deeply connected. We investigate the connection between the log ML and the sum of predictive log likelihoods of datapoints, conditioned on preceding data in the dataset. This perspective reveals a family of estimators of the log ML which depend only on predictions sampled from the posterior of an iterative Bayesian updating procedure. 
We study the proposed estimator family in the context of linear models, where we can conclusively analyze its theoretical properties.
Leveraging the fact that gradient descent can produce exact posterior samples for linear models \citep{matthews2017} and the infinite-width limit of deep neural networks \citep{matthews2018gaussian,lee2018deep}, we show that this estimator can be viewed as the sum of a subset of the model's training losses in an iterative optimization procedure. 
This immediately yields an interpretation of marginal likelihood estimation as measuring a notion of training speed in linear models. We further show that this notion of training speed is predictive of the weight assigned to a model in a linear model combination trained with gradient descent, hinting at a potential explanation for the bias of gradient descent towards models that generalize well in more complex settings.

We demonstrate the utility of the estimator through empirical evaluations on a range of model selection problems, confirming that it can effectively approximate the marginal likelihood of a model. Finally, we empirically evaluate whether our theoretical results for linear models may have explanatory power for more complex models. We find that an analogue of our estimator for DNNs trained with stochastic gradient descent (SGD) is predictive of both final test accuracy and the final weight assigned to the model after training a linear model combination. Our findings in the deep learning setting hint at a promising avenue of future work in explaining the empirical generalization performance of DNNs.

\section{Background and Related Work}

\subsection{Bayesian Parameter Inference}
\svsp
A Bayesian model $\model$ is defined by a prior distribution over parameters $\theta$, $P(\theta | \model)$, and a prediction map from parameters $\theta$ to a likelihood over the data $\data$, $P(\data|\theta, \model)$.
Parameter fitting in the Bayesian framework entails finding the posterior distribution $P(\theta|\data)$, which yields robust and principled uncertainty estimates. 
Though exact inference is possible for certain models like Gaussian processes (GPs) \citep{rasmussen2003gaussian}, it is intractable for DNNs. Here approximations such as variational inference \citep{blei2017variational} are used \citep{gal2016dropout, blundell2015weight,mackay1992bayesian, graves2011practical, duvenaud2016early}, to improve robustness and obtain useful uncertainty estimates.

Variational approximations require optimisation over the parameters of the approximate posterior distribution. 
This optimization over distributions changes the loss landscape, and is significantly slower than the pointwise optimization used in standard DNNs.
Pointwise optimization methods inspired by Bayesian posterior sampling can produce similar variation and uncertainty estimates as variational inference, while improving computational efficiency \citep{welling2011bayesian,mandt2017stochastic,maddox2019simple}.
An appealing example of this is ensembling \citep{lakshminarayanan2017simple}, which works by training a collection models in the usual pointwise manner, starting from $k$ independently initialized points.

In the case of linear models, this is exactly equivalent to Bayesian inference, as this sample-then-optimize approach yields exact posterior samples \citep{matthews2017, osband2018randomized}. \citet{he2020bayesian} extend this approach to obtain posterior samples from DNNs in the infinite-width limit.

\subsection{Bayesian Model Selection}
\svsp

In addition to finding model parameters, Bayesian inference can also perform \textit{model selection} over different inductive biases, which are specified through both model structure (e.g.~convolutional vs fully connected) and the prior distribution on parameters. The Bayesian approach relies on finding the posterior over models $P(\model|\data)$, which uses the \emph{marginal likelihood} (ML) as its likelihood function:
\begin{equation}\label{eq:marg-lik}
P(\data | \model) = \int_\theta P(\data|\theta)P(\theta|\model_i)d\theta = \mathbb{E}_{P(\theta|\model)} P(\data | \theta)\,.
\end{equation}
Instead of computing the full posterior, it is common to select the model with the highest marginal likelihood. This is known as type-II maximum likelihood \citep{mackay1992bayesian,mackay2003information} and is less prone to overfitting than performing maximum likelihood over the parameters and model combined. This is because the marginal likelihood is able to trade off between model fit and model complexity \citep{rasmussen2001occam}.
Maximising the ML is standard procedure when it is easy to compute. For example, in Gaussian processes it used to set simple model parameters like smoothness \citep{rasmussen2003gaussian}, while recent work has demonstrated that complex inductive biases in the form of invariances can also be learned \citep{van_der_wilk_learning_2018}.

For many deep models, computing Equation~\ref{eq:marg-lik} is intractable, and obtaining approximations that are accurate enough for model selection and that scale to complex models is an active area of research \cite{khan2019approximate}. In general, variational lower bounds that scale are too loose when applied to DNNs \citep{blundell2015weight}. Deep Gaussian processes provide a case where the bounds do work \citep{damianou13a,dutordoir20a}, but heavy computational load holds performance several years behind deep learning. While ensembling methods provide useful uncertainty estimates and improve the computational efficiency of the variational approach, they have not yet provided a solution for Bayesian model selection.

\subsection{Generalization and Risk Minimization}
\svsp
Bayesian model selection addresses a subtly different problem from the risk minimization framework used in many learning problems. Nonetheless, the two are closely related; \citet{germain_pac-bayesian_2016} show that in some cases optimizing a PAC-Bayesian risk bound is equivalent to maximizing the marginal likelihood of a Bayesian model. In practice, maximizing an approximation of the marginal likelihood in DNNs trained with SGD can improve generalization performance \citep{smith2017bayesian}.
More recently, \citet{arora2019fine} computed a data-dependent complexity measure which resembles the data-fit term in the marginal likelihood of a Bayesian model and which relates to optimization speed, hinting at a potential connection between the two. 

At the same time, generalization in deep neural networks (DNNs) remains mysterious, with classical learning-theoretic bounds failing to predict the impressive generalization performance of DNNs \citep{zhang2016understanding, nagarajan2019uniform}. Recent work has shown that DNNs are biased towards functions that are `simple', for various definitions of simplicity \citep{kalimeris2019sgd, frankle2018the, valle2018deep, smith2018}. PAC-Bayesian generalization bounds, which can quantify a broad range of definitions of complexity, can attain non-vacuous values \citep{ mcallester1999, dziugaite_computing_2017, dziugaite_data-dependent_2018}, but nonetheless exhibit only modest correlation with generalization error \citep{jiang2019fantastic}. These bounds depend only on the final distribution over parameters after training; promising alternatives consider properties of the trajectory taken by a model during optimization \citep{hardt2015train, negrea2019information}. This trajectory-based perspective is a promising step towards explaining the correlation between the number of training steps required for a model to minimize its objective function and its final generalization performance observed in a broad range of empirical analyses \citep{jiang2019fantastic, belkin2018reconciling, nakkiran2019deep, ru2020revisiting}.  

\section{Marginal Likelihood Estimation with Training Statistics}\label{sec:ml-estimation}
In this section, we investigate the equivalence between the marginal likelihood (ML) and a notion of training speed in models trained with an exact Bayesian updating procedure. For linear models and infinitely wide neural networks, exact Bayesian updating can be done using gradient descent optimisation. For these cases, we derive an estimator of the marginal likelihood which
\begin{inparaenum}[\bgroup\bfseries 1)\egroup]
\item is related to how quickly a model learns from data,
\item only depends on statistics that can be measured during pointwise gradient-based parameter estimation, and
\item becomes tighter for ensembles consisting of multiple parameter samples.
\end{inparaenum} We also investigate how gradient-based optimization of a linear model combination can implicitly perform approximate Bayesian model selection in Section~\ref{sec:optimize-then-prune}.

\subsection{Training Speed and the Marginal Likelihood} \label{sec:decomposing_ML}
\svsp
Let $\data$ denote a dataset of the form $\data = (\data_i)_{i=1}^n = (x_i, y_i)_{i=1}^n$, and let $\data_{<i}=(\data_j)_{j=1}^{i-1}$ with $\data_{<1}=\emptyset$.
We will abbreviate $P(\trdata|\model) := P(\trdata)$ when considering a single model $\model$. Observe that $P(\trdata) = \prod_{i=1}^n P(\trdata_i|\trdata_{<i})$ to get the following form of the \textit{log} marginal likelihood:

\begin{equation}
    \log P(\data) = \log \prod_{i=1}^n P(\data_i|\data_{<i}) = \sum_{i=1}^n \log P(\data_i | \data_{<i}) = \sum_{i=1}^n \log [\mathbb{E}_{P(\theta|\data_{<i})} P(\data_i|\theta) ].
\end{equation}

If we define training speed as the number of data points required by a model to form an accurate posterior, then models which train faster -- i.e. whose posteriors assign high likelihood to the data after conditioning on only a few data points -- will obtain a higher marginal likelihood. Interpreting the negative log posterior predictive probability $\log P(\data_i|\data_{<i})$ of each data point as a loss function, the log ML then takes the form of the sum over the losses incurred by each data point during training, i.e. the area under a training curve defined by a Bayesian updating procedure. 

\subsection{Unbiased Estimation of a Lower Bound} \label{sec:LB}
\svsp
In practice, computing $\log P(\data_i|\data_{<i})$ may be intractable, necessitating approximate methods to estimate the model evidence. In our analysis, we are interested in estimators of $\log P(\data)$ computed by drawing $k$ samples of $\theta \sim P(\theta|\data_{<i})$ for each $i=1, \dots, n$. We can directly estimate a lower bound $\mathcal{L}(\data) = \sum_{i=1}^n\mathbb{E}[\log P(\data_i|\data_{<i})$ using the log likelihoods of these samples
\begin{equation}
    \hat{\mathcal{L}}(\data) = \sum_{i=1}^n \frac{1}{k}\sum_{j=1}^k\log P(\data_{i}|\theta^i_j).
\end{equation}
This will produce a biased estimate of the log marginal likelihood due to Jensen's inequality. We can get a tighter lower bound by first estimating $\mathbb{E}[\log P(\data_i|\theta)]$ using our posterior samples before applying the logarithm, obtaining

\begin{equation}\label{eq:estimators}
    \hat{\mathcal{L}}_k(\data) = \sum_{i=1}^n \log \frac{1}{k}\sum_{j=1}^k P(\data_{i}|\theta^i_j).
\end{equation} 

\begin{restatable}{proposition}{PropLk}\label{prop:lk}
Both $\hat{\mathcal{L}}$ and $\hat{\mathcal{L}}_k$ as defined in Equation~\ref{eq:estimators} are estimators of lower bounds on the log marginal likelihood; that is
\begin{equation}
   \mathbb{E}[\hat{\mathcal{L}}(\data)] = \mathcal{L}(\data) \leq  \log P(\data) \quad \text{ and } \quad 
    \mathbb{E}[\hat{\mathcal{L}}_k(\data)] = \mathcal{L}_k(\data) \leq \log P(\data) \; .
\end{equation}
Further, the bias term in $\mathcal{L}$ can be quantified as follows.
\begin{equation}\mathcal{L}(\data) = \log P(\data) - \sum_{i=1}^n \KL(P(\theta | \data_{<i})|| P(\theta|\data_{< {i+1}}))
\end{equation}
\end{restatable}
We include the proof of this and future results in Appendix~\ref{sec:proofs}. 
We observe that both lower bound estimators exhibit decreased variance when using multiple posterior samples; however, $\hat{\mathcal{L}}_k$ also exhibits decreasing bias (with respect to the log ML) as $k$ increases; each $k$ defines a distinct lower bound $\mathcal{L}_k = \mathbb{E}[\hat{\mathcal{L}}_k ]$ on $\log P(\data)$. The gap induced by the lower bound $\mathcal{L}(\data)$ is characterized by the
information gain each data point provides to the model about the posterior, as given by the Kullback-Leibler (KL) divergence \citep{kullback1951} between the posterior at time $i$ and the posterior at time $i+1$. Thus, while $\mathcal{L}$ has a Bayesian interpretation it is arguably more closely aligned with the minimum description length notion of model complexity \citep{hinton1993keeping}.

When the posterior predictive distribution of our model is Gaussian, we consider a third approach which, unlike the previous two methods, also applies to noiseless models. Let $\trdata =(X_i, Y_i)_{i=1}^n$, and $(\theta^i_j)_{j=1}^k$ be $k$ parameter samples from $P(\theta|\data_{<i})$. We assume a mapping $f: \Theta \times X \rightarrow Y$ such that sampling parameters $\theta$ and computing $f(\theta, X_i)$ is equivalent to sampling from the posterior $P(\cdot|\data_{<i}, X_i)$. We can then obtain the following estimator of a lower bound on $\log \mathcal{P}(\data)$.

\begin{restatable}{proposition}{PropLS}\label{prop:ls}
Let $P(Y_i|\data_{<i}, X_i) = \mathcal{N}(\mu_i, \sigma^2_i)$ for some $\mu_i, \sigma_i^2$. Define the standard mean and variance estimators $\hat{\mu}_i = \frac{1}{N} \sum_{j=1}^N f(\theta^i_j, x_i)$ and $\hat{\sigma}^2_i = \frac{1}{N-1} \sum (f(\theta_{j}^i, x_i) - \hat{\mu})^2$. Then the  estimator
\begin{equation}
 \hat{\mathcal{L}}_S(\data) = \sum_{i=1}^n \log P(Y_i|\hat{\mu}_i, \hat{\sigma}^2_i) 
\end{equation}
is a lower bound on the log ML: i.e. $\mathbb{E}[\hat{\mathcal{L}}_S(\data)] \leq \log P(\data) $. 
\end{restatable}

We provide an empirical evaluation of the rankings provided by the different estimators in Section~\ref{sec:experiments}. We find that $\hat{\mathcal{L}}_S$ exhibits the least bias in the presence of limited samples from the posterior, though we emphasize its limitation to Gaussian posteriors; for more general posterior distributions, $\hat{\mathcal{L}}_k$ minimizes bias while still estimating a lower bound. 

\subsubsection{Lower bounds via gradient descent trajectories} \label{sec:LBviaGD}
\ssvsp
The bounds on the marginal likelihood we introduced in the previous section required samples from the sequence of posteriors as data points were incrementally added $p(\theta|\data_{<i})$.
Ensembles of linear models trained with gradient descent yield samples from the model posterior. We now show that we can use these samples to estimate the log ML using the estimators introduced in the previous section.

We will consider the Bayesian linear regression problem of modelling data $\data = (X_i, Y_i)_{i=1}^n$ assumed to be generated by the process $Y = \theta^\top \Phi(X) + \epsilon \sim \mathcal{N}(0, \sigma_N^2 I)$ for some unknown $\theta$, known $\sigma_N^2$, and feature map $\Phi$. Typically, a Gaussian prior is placed on $\theta$; this prior is then updated as data points are seen to obtain a posterior over parameters. In the overparmeterised, noiseless linear regression setting, \citet{matthews2017} show that the distribution over parameters $\theta$ obtained by sampling from the prior on $\theta_0$ and running gradient descent to convergence on the data $\data_{<i}$ is equivalent to sampling from the posterior conditioned on $\data_{<i}$. \citet{osband2018randomized} extend this result to posteriors which include observation noise $\sigma^2_N \neq 0$ under the assumption that the targets $Y_i$ are themselves noiseless observations. 

\begin{algorithm} \label{alg:estimate}
\SetAlgoLined
\KwIn{A dataset $\data =(x_i, y_i)_{i=1}^n $, parameters $\mu_0, \sigma_0^2, \sigma_N^2$}
\KwResult{An estimate of $\mathcal{L}(\data)$}
$\theta_t \gets \theta_0 \sim \mathcal{N}(\mu_0, \sigma_0^2)$; \quad 
$\tilde{Y} \gets Y + \epsilon \sim \mathcal{N}(0, \sigma_N^2)$;  \quad sumLoss $\gets$ 0 \; $\ell(\data_{\le i}, w) \gets \|\tilde{Y}_{\le i} - \theta^\top X_{\le i} \|_2^2 + \frac{\sigma_N^2}{\theta_0^2}\|\theta - \theta_0\|_2^2 $\;

 \For{$\data_i \in \data$}{
  sumLoss $ = $ sumLoss $ + \; \frac{(\theta_t^\top x_i - y_i)^2}{2\sigma_N^2}$ \;
  $\theta_t \gets$ GradientDescent($ \ell, \theta_t, \data_{\le i}$) \;
 }
 \KwRet sumLoss
 \caption{Marginal Likelihood Estimation for Linear Models}
\end{algorithm}

We can use this procedure to obtain posterior samples for our estimators by iteratively running sample-then-optimize on the sets $\data_{<i}$. Algorithm \ref{alg:estimate} outlines our approach, which uses sample-then-optimize on iterative subsets of the data to obtain the necessary posterior samples for our estimator. Theorem \ref{thm:sto} shows that this procedure yields an unbiased estimate of $\mathcal{L}(\data)$ when a single prior sample is used, and an unbiased estimate of $\mathcal{L}_k(\data)$ when an ensemble of $k$ models are trained in parallel.

\begin{restatable}{theorem}{ThmSTO} \label{thm:sto}
Let $\data = (X_i, Y_i)_{i=1}^n$ and let $(\theta_j^i)_{i,j=1}^{n,J}$ be generated by the procedure outlined above. Then the estimators $\hat{\mathcal{L}}, \hat{\mathcal{L}}_S,$ and $ \hat{\mathcal{L}}_k$, applied to the collection $(\theta_j^i)$, are lower bounds on $\log P(\data)$. Further, expressing $-\log P(\data_i|\theta)$ as the $\ell_2$ regression loss plus a constant, we then obtain 
\begin{equation}
    \log P(\data) \geq \sum_{i=1}^n \mathbb{E}_{\theta_i \sim P(\cdot | \data_{<i})}[\log P(\data_i|\theta_i)] = \mathbb{E}\sum_{i=1}^n -\ell_2 (\data_i, \theta_i) + c = \mathcal{L}(\data)
\end{equation}
\end{restatable}

We highlight that Theorem \ref{thm:sto} precisely characterizes the lower bound on the marginal likelihood as a sum of `training losses' based on the regression loss $\ell_2(\data_i, \theta_i)$. 

\subsubsection{From Linear Models to Infinite Neural Networks}
\ssvsp

Beyond linear models, our estimators can further perform model selection in the infinite-width limit of neural networks. Using the optimization procedure described by \citet{he2020bayesian}, we can obtain an exact posterior sample from a GP given by the neural tangent kernel \citep{jacot2018neural}. The iterative training procedure described in Algorithm~\ref{alg:estimate} will thus yield a lower bound on the marginal likelihood of this GP using sampled losses from the optimization trajectory of the neural network. We evaluate this bound in Section \ref{sec:experiments}, and formalize this argument in the following corollary. 
\begin{restatable}{corollary}{CorNTK}\label{cor:ntk}
Let $\trdata$ be a dataset indexed by our standard notation. Let $f_0$ be sampled from an infinitely wide neural network architecture $\mathcal{F}$ under some initialization distribution, and let $f_\infty^i$ be the limiting solution under the training dynamics defined by \citet{he2020bayesian} applied to the initialization $f_0$ and using data $\trdata_{< i}$. Let $K_\infty$ denote the neural tangent kernel for $\mathcal{F}$, and $\mathcal{M}=GP(0, K_\infty)$ the induced Gaussian Process. Then $f_\infty^i \sim P(f|\trdata_{< i}, \model)$, and in the limit of infinite training time, the iterative sample-then-optimize procedure yields an unbiased estimate of $\mathcal{L}(\trdata |\model)$. Letting $\ell_2$ denote the scaled squared $\ell_2$ regression loss and $c$ be a constant, we obtain as a direct corollary of Theorem~\ref{thm:sto}
\begin{equation}
    P(\data) \geq \mathbb{E}_{f_\infty^i \sim P(\cdot | \data_{<i})}[\log P(\data_i|\theta_i)] = \mathbb{E}\sum_{i=1}^n -\ell_2 (\data_i, f_i) + c = \mathcal{L}(\data) \; .
\end{equation}
\end{restatable}

This result provides an additional view on the link between training speed and generalisation in wide neural networks noted by \citet{arora2019fine}, who analysed the convergence of gradient descent. They compute a PAC generalization bound which a features the data complexity term equal to that in the marginal likelihood of a Gaussian process \citet{rasmussen2003gaussian}. This term provides a bound on the rate of convergence of gradient descent, whereas our notion of training speed is more closely related to sample complexity and makes the connection to the marginal likelihood more explicit.

It is natural to ask if such a Bayesian interpretation of the sum over training losses can be extended to non-linear models trained with stochastic gradient descent.  Although SGD lacks the exact posterior sampling interpretation of our algorithm, we conjecture a similar underlying mechanism connecting the sum over training losses and generalization. Just as the marginal likelihood measures how well model updates based on previous data points generalize to a new unseen data point, the sum of training losses measures how well parameter updates based on one mini-batch generalize to the rest of the training data. If the update generalizes well, we expect to see a sharper decrease in the training loss, i.e. for the model to train more quickly and exhibit a lower sum over training losses. This intuition can be related to the notion of `stiffness' proposed by \citet{fort2019stiffness}. We provide empirical evidence supporting our hypothesis in Section \ref{sec:DNN_exp}.
\subsection{Bayesian Model Selection and Optimization}\label{sec:optimize-then-prune}
\svsp
The estimator $\mathcal{L}(\data)$ reveals an intriguing connection between pruning in linear model combinations and Bayesian model selection. We assume a data set $\data = (X_i, Y_i)_{i=1}^n$ and a collection of $k$ models $\model_1, \dots, \model_k$. A linear regressor $w$ is trained to fit the posterior predictive distributions of the models to the target $Y_i$; i.e. to regress on the dataset 
\begin{equation}
    (\Phi, Y) = \bigg (\phi_i=(\hat{Y}^i_1, \dots, \hat{Y}_n^i), Y_i\bigg)_{i=1}^n \text{ with } \hat{Y}_j^i \sim P(\hat{Y}|\data_{<i}, X_i, \model_j).
\end{equation}
The following result shows that the optimal linear regressor on this data generating distribution assigns the highest weight to the model with the highest $\mathcal{L}(\data)$ whenever the model errors are independent. This shows that magnitude pruning in a linear model combination is equivalent to approximate Bayesian model selection, under certain assumptions on the models.

\begin{restatable}{proposition}{PropBMS}\label{prop:modelselect}
Let $\model_1, \dots, \model_k$ be Bayesian linear regression models with fixed noise variance $\sigma_N^2$ and Gaussian likelihoods. Let $\Phi$ be a (random) matrix of posterior prediction samples, of the form $\Phi[i, j] = \hat{y}_i^j \sim P(y_j|\data_{<j}, x_j, \model_i)$. Suppose the following two conditions on the columns of $\Phi$ are satisfied: $\mathbb{E}\langle \Phi[:, i], y \rangle = \mathbb{E}\langle \Phi[:, j], y \rangle$ for all $i, j$, and $\mathbb{E}\langle \Pi_{y^\perp} \phi_i, \Pi_{y^\perp} \phi_j \rangle = 0$. Let $w^*$ denote the least-squares solution to the regression problem $\min_w \mathbb{E}_{\Phi}\|\Phi w - y\|^2$. Then the following holds
\begin{equation}  \argmax_i w^*_i = \argmax_i \mathcal{L}(\data | \model_i)  \qquad \forall w^* = \argmin_w \mathbb{E} \|\Phi w - y\|^2\;. \end{equation}
\end{restatable}

The assumption on the independence of model errors is crucial in the proof of this result: families of models with large and complementary systematic biases may not exhibit this behaviour. We observe in Section \ref{sec:experiments} that the conditions of Proposition 1 are approximately satisfied in a variety of model comparison problems, and running SGD on a linear combination of Bayesian models still leads to solutions that approximate Bayesian model selection.
We conjecture that analogous phenomena occur during training within a neural network. The proof of Proposition~\ref{prop:modelselect} depends on the observation that, given a collection of features, the best least-squares predictor will assign the greatest weight to the feature that best predicts the training data. While neural networks are not linear ensembles of fixed models, we conjecture that, especially for later layers of the network, a similar phenomenon will occur wherein weights from nodes that are more predictive of the target values over the course of training will be assigned higher magnitudes. We empirically investigate this hypothesis in Section \ref{sec:DNN_exp}.

\section{Empirical Evaluation} \label{sec:experiments}
Section \ref{sec:ml-estimation} focused on two key ideas: that training statistics can be used as an estimator for a Bayesian model's marginal likelihood (or a lower bound thereof), and that gradient descent on a linear ensemble implicitly arrives at the same ranking as this estimator in the infinite-sample, infinite-training-time limit. We further conjectured that similar phenomena may also hold for deep neural networks. We now illustrate these ideas in a range of settings. Section \ref{sec:BMS} provides confirmation and quantification of our results for linear models, the model class for which we have theoretical guarantees, while Section \ref{sec:DNN_exp} provides preliminary empirical confirmation that the mechanisms at work in linear models also appear in DNNs. 

\subsection{Bayesian Model Selection}
\label{sec:BMS}
\svsp
While we have shown that our estimators correspond to lower bounds on the marginal likelihood, we would also like the relative rankings of models given by our estimator to correlate with those assigned by the marginal likelihood. We evaluate this correlation in a variety of linear model selection problems. We consider three model selection problems; for space we focus on one, feature dimension selection, and provide full details and evaluations on the other two tasks in Appendix \ref{sec:ex_ms_blr_synthetic_data}.

For the feature dimension selection task, we construct a synthetic dataset inspired by \citet{wilson2020bayesian} of the form $(\textbf{X}, \textbf{y})$, where $x_i = (y_i + \epsilon_1,  y_i + \dots, y_i + \epsilon_{15}, \epsilon_{16}, \dots, \epsilon_{30})$, and consider a set of models $\{\model_k\}$ with feature embeddings $\phi_k(x_i) = x_i[1, \dots, k]$. The optimal model in this setting is the one which uses exactly the set of `informative' features $x[1, \dots, 15]$. 

We first evaluate the relative rankings given by the true marginal likelihood with those given by our estimators. We compare $\mathcal{L}_S$, $\mathcal{L}$ and $\mathcal{L}_k$; we first observe that all methods agree on the optimal model: this is a consistent finding across all of the model selection tasks we considered. While all methods lower bound the log marginal likelihood, $\mathcal{L}_k(\data)$ and $\mathcal{L}_S(\data)$ exhibit a reduced gap compared to the naive lower bound.
In the rightmost plot of Figure~\ref{fig:lse-estimator}, we further quantify the reduction in the bias of the estimator $\mathcal{L}_k(\data)$ described in Section~\ref{sec:ml-estimation}. We use exact posterior samples (which we denote in the figure simply as posterior samples) and approximate posterior samples generated by the gradient descent procedure outlined in Algorithm~\ref{alg:estimate} using a fixed step size and thus inducing some approximation error. We find that both sampling procedures exhibit decreasing bias as the number of samples $k$ is increased, with the exact sampling procedure exhibiting a slightly smaller gap than the approximate sampling procedure.

\begin{figure}
    \centering
    \includegraphics[width=0.47\linewidth]{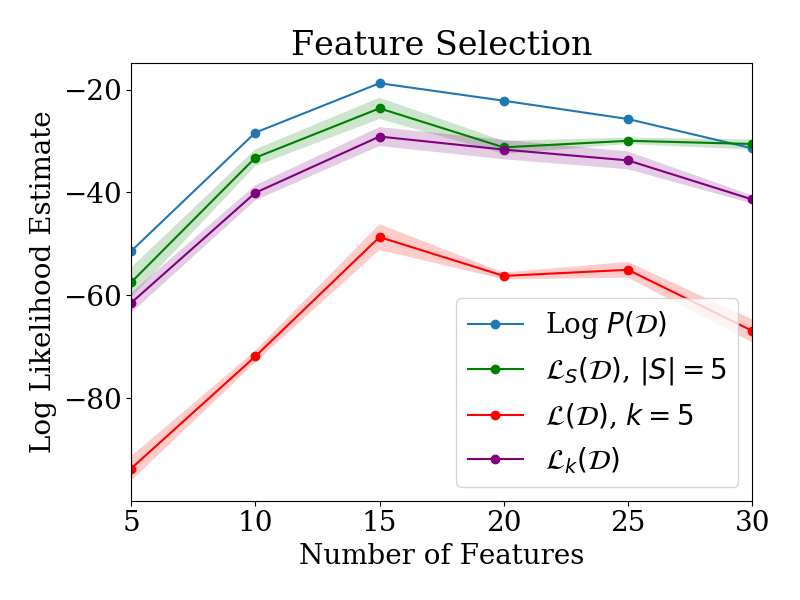}
    \includegraphics[width=0.47\linewidth]{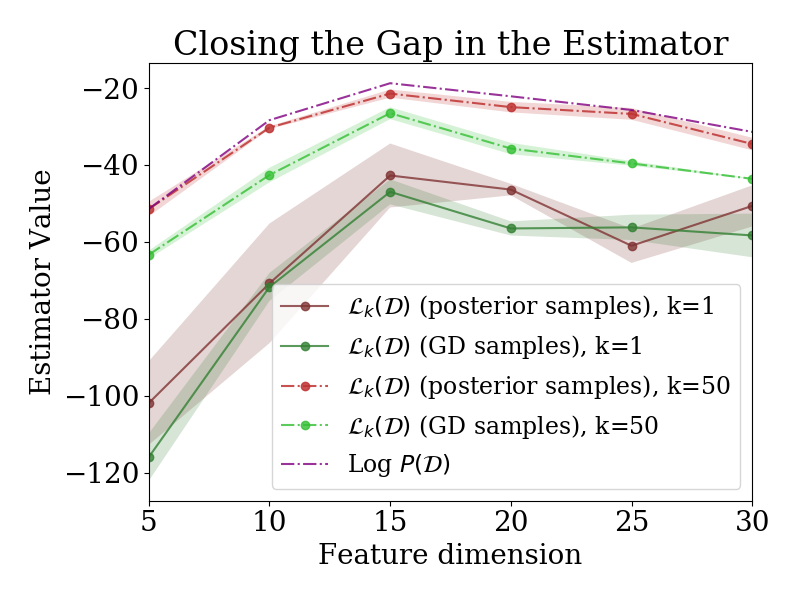}
    \caption{Left: ranking according to $\log P(\data)$, $\mathcal{L}(\data)$ with exact posterior samples, and $\mathcal{L}(\data)$ computed on samples generated by gradient descent. Right: gap between true marginal likelihood and $\mathcal{L}_k(\data)$ estimator shrinks as a function of $k$ for both exact and gradient descent-generated samples. }
    \label{fig:lse-estimator}
    \vspace{-5pt}
\end{figure}

We next empirically evaluate the claims of Proposition~\ref{prop:modelselect} in settings with relaxed assumptions. We compare the ranking given by the true log marginal likelihood, the estimated $\mathcal{L}(\data)$, and the weight assigned to each model by the trained linear regressor. We consider three variations on how sampled predictions from each model are drawn to generate the features $\phi_i$: sampling the prediction for point $\hat{Y}_i$ from $P(\hat{Y}_i | \data_{<i})$ (`concurrent sampling' -- this is the setting of Proposition \ref{prop:modelselect}), as well as two baselines: the posterior $P(\hat{Y}_i |\data)$ (`posterior sampling'), and  the prior $P(\hat{Y}_i)$ (`prior sampling'). We find that the rankings of the marginal likelihood, its lower bound, and of the ranking given by concurrent optimization all agree on the best model in all three of the model selection problems outlined previously, while the prior and posterior sampling procedure baselines do not exhibit a consistent ranking with the log ML. We visualize these results for the feature dimension selection problem in Figure~\ref{fig:ml_v_weight}; full results are shown in Figure \ref{fig:app_ml_v_weight}.

\begin{figure}
    \centering
    \includegraphics[width=0.38\linewidth]{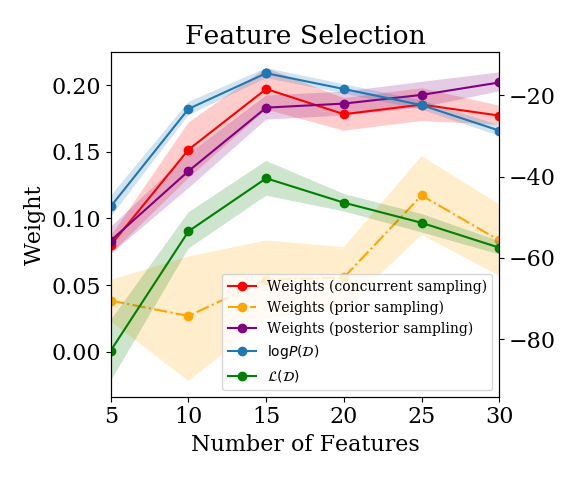}
    \includegraphics[width=0.44\linewidth]{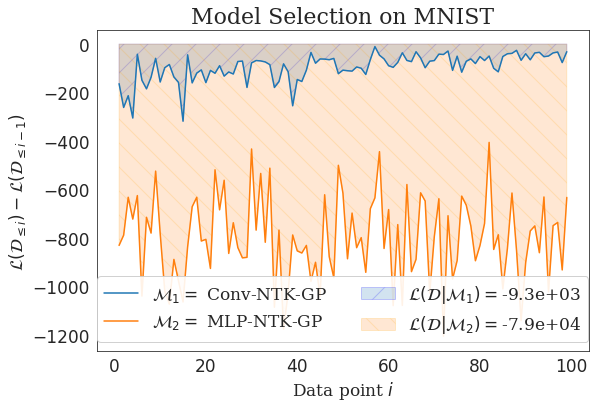}
    \hfill 
    \caption{Left: Relative rankings given by optimize-then-prune, ML, and estimated $\mathcal{L}(\data)$ on the feature selection problem. Right: visualizing the interpretation of $\mathcal{L}(\data)$ as the `area under the curve' of training losses: we plot the relative change in the estimator  $\mathcal{L}(\data_{\le i}) - \mathcal{L}(\data_{<i}) $ for convolutional and fully-connected NTK-GP models, and shade their area.
    }
    \label{fig:ml_v_weight}
\end{figure}
We further illustrate how the $\mathcal{L}(\data)$ estimator can select inductive biases in the infinite-width neural network regime in Figure~\ref{fig:ml_v_weight}. Here we evaluate the relative change in the log ML of a Gaussian Process induced by a fully-connected MLP (MLP-NTK-GP) and a convolutional neural network (Conv-NTK-GP) which performs regression on the MNIST dataset.
The fully-connected model sees a consistent decrease in its log ML with each additional data point added to the dataset, whereas the convolutional model sees the incremental change in its log ML become less negative as more data points are added as a result of its implicit bias, as well as a much higher incremental change in its log ML from the start of training.
This leads to the Conv-NTK-GP having a higher  value for $\mathcal{L}(\data)$ than the MLP-NTK-GP. We provide an analogous plot evaluating $\log P(\data)$ in the appendix.

\subsection{Training Speed, Ensemble Weight, and Generalization in DNNs} \label{sec:DNN_exp}
\svsp

We now address our conjectures from Section 3, which aim to generalize our results for linear models to deep neural networks trained with SGD. Recall that our hypothesis involves translating \textit{iterative posterior samples} to \textit{minibatch training losses over an SGD trajectory}, and \textit{bayesian model evidence} to \textit{generalization error}; we conjectured that just as the sum of the log posterior likelihoods is useful for Bayesian model selection, the sum of minibatch training losses will be useful to predict generalization error. In this section, we evaluate whether this conjecture holds for a simple convolutional neural network trained on the FashionMNIST dataset. Our results provide preliminary evidence in support of this claim, and suggest that further work investigating this relationship may reveal valuable insights into how and why neural networks generalize.  
\subsubsection{Linear Combination of DNN Architectures} \label{sec:sgd_dnn}
\ssvsp

We first evaluate whether the sum over training losses (SOTL) obtained over an SGD trajectory correlates with a model's generalization error, and whether SOTL predicts the weight assigned to a model by a linear ensemble. To do so, we train a linear combination of DNNs with SGD to determine whether SGD upweights NNs that generalize better. Further details of the experiment can be found in Appendix \ref{sec:exp_details_sgd_dnn}. Our results are summarized in Figure \ref{fig:mod_select_dnn}.

\begin{figure}
    \begin{minipage}{.27\textwidth}
    \includegraphics[ width=\linewidth]{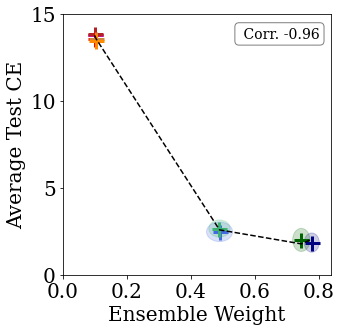}
    \end{minipage}
    \begin{minipage}{.27\textwidth}
    \includegraphics[ width=\linewidth]{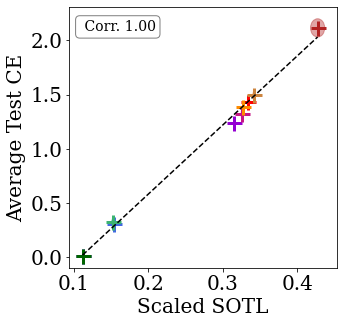}
    \end{minipage}
    \begin{minipage}{.27\textwidth}
    \includegraphics[ width=\linewidth]{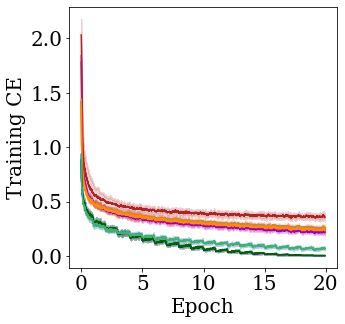}
    \end{minipage}
    \begin{minipage}{.17\textwidth}
    \includegraphics[ width=\linewidth]{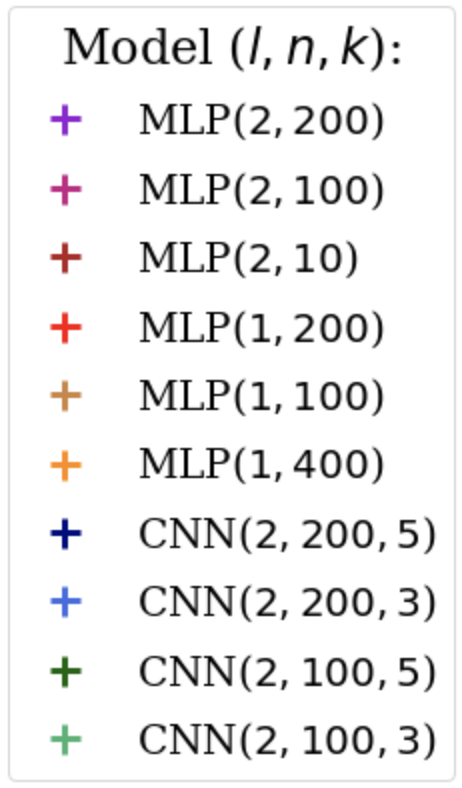}
    \end{minipage}
    \caption{Linear combinations of DNNs on FashionMNIST trained. Left: ensemble weights versus the test loss for concurrent training. Middle: sum over training losses (SOTL), standardized by the number of training samples, versus test loss for parallel training. Right: training curves for the different models trained in parallel. All results are averaged over $10$ runs, and standard deviations are shown by the shaded regions around each observation. The model parameters, given in the parentheses, are the number of layers ($l$), nodes per layer ($n$) and kernel size ($k$), respectively. }
    \label{fig:mod_select_dnn}
\end{figure}

We observe a strong correlation between SOTL and average test cross-entropy (see Figure \ref{fig:mod_select_dnn} middle column), validating that the SOTL is correlated with generalization. Further, we find that architectures with lower test error (when trained individually) are given higher weight by the linear ensembling layer -- as can be seen from the left plot in Figure \ref{fig:mod_select_dnn}. This supports our hypothesis that \textit{SGD favours models that generalize well}.

\subsubsection{Subnetwork Selection in Neural Networks} \label{sec:sgd_submodel}
\ssvsp
Finally, we evaluate whether our previous insights apply to submodels within a neural network, suggesting a potential mechanism which may bias SGD towards parameters with better generalization performance. Based on the previous experiments, we expect that nodes that have a lower sum over training errors (if evaluated as a classifier on their own) are favoured by gradient descent and therefore have a larger final weight than those which are less predictive of the data. If so, we can then view SGD followed by pruning (in the final linear layer of the network) as performing an approximation of a Bayesian model selection procedure. We replicate the model selection problem of the previous setting, but replace the individual models with the activations of the penultimate layer of a neural network, and replace the linear ensemble with the final linear layer of the network.  Full details on the experimental set-up can be found in Appendix \ref{sec:exp_details_sgd_submodels}. We find that our hypotheses hold here: SGD assigns larger weights to subnetworks that perform well, as can be seen in Figure \ref{fig:sgd_submodel}. This suggests that SGD is biased towards functions that generalize well, even within a network. We find the same trend holds for CIFAR-10, which is shown in Appendix \ref{sec:exp_details_sgd_submodels}.

\begin{figure}
    \centering
    \includegraphics[ width=\linewidth]{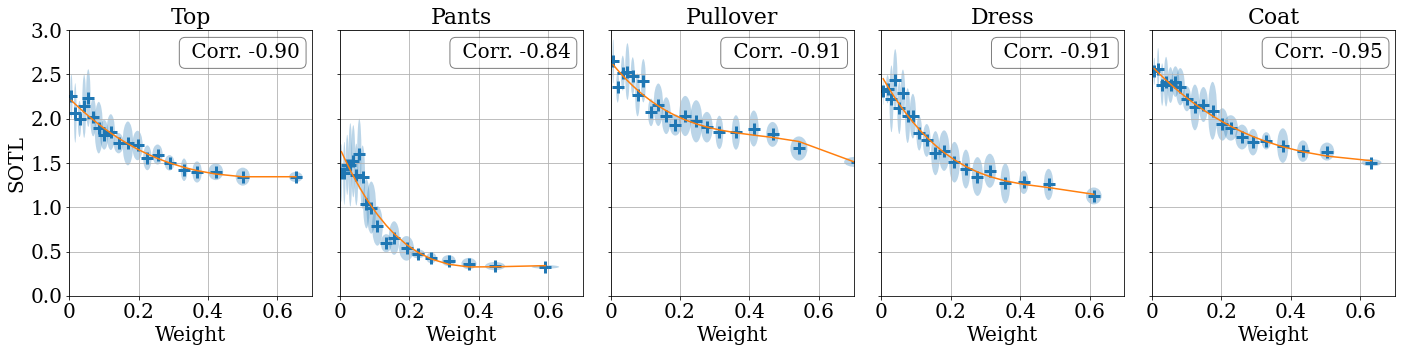}
    \caption{Weight assigned to subnetwork by SGD in a deep neural network (x-axis) versus the subnetwork performance (estimated by the sum of cross-entropy, on the y-axis) for different FashionMNIST classes. The light blue ovals denote depict $95\%$ confidence intervals, estimated over 10 seeds (i.e. 2$\sigma$ for both the weight and SOTL).  The orange line depicts the general trend.}
    \label{fig:sgd_submodel}
\end{figure}

\section{Conclusion}

In this paper, we have proposed a family of estimators of the marginal likelihood which illustrate the connection between training speed and Bayesian model selection. Because gradient descent can produce exact posterior samples in linear models, our result shows that Bayesian model selection can be done by training a linear model with gradient descent and tracking how quickly it learns. This approach also applies to the infinite-width limit of deep neural networks, whose dynamics resemble those of linear models.
We further highlight a connection between magnitude-based pruning and model selection, showing that models for which our lower bound is high will be assigned more weight by an optimal linear model combination.
This raises the question of whether similar mechanisms exist in finitely wide neural networks, which do not behave as linear models.
We provide preliminary empirical evidence that the connections shown in linear models have predictive power towards explaining generalization and training dynamics in DNNs, suggesting a promising avenue for future work.

\newpage 
\section{Broader Impact}
Due to the theoretical nature of this paper, we do not foresee any immediate applications (positive or negative) that may arise from our work. However, improvement in our understanding of generalization in deep learning may lead to a host of downstream impacts which we outline briefly here for completeness, noting that the marginal effect of this paper on such broad societal and environmental impacts is likely to be very small.
\begin{enumerate}
    \item \textbf{Safety and robustness.} Developing a stronger theoretical understanding of generalization will plausibly lead to training procedures which improve the test-set performance of deep neural networks. Improving generalization performance is crucial to ensuring that deep learning systems applied in practice behave as expected based on their training performance. 
    \item \textbf{Training efficiency and environmental impacts.} In principle, obtaining better estimates of model and sub-model performance could lead to more efficient training schemes, thus potentially reducing the carbon footprint of machine learning research.
    \item \textbf{Bias and Fairness.} The setting of our paper, like much of the related work on generalization, does not consider out-of-distribution inputs or training under constraints. If the training dataset is biased, then a method which improves the generalization performance of the model under the i.i.d. assumption will be prone to perpetuating this bias. 
\end{enumerate}

\section*{Acknowledgements}
Lisa Schut was supported by the Accenture Labs and Alan Turing Institute. 

\newpage 
\bibliographystyle{plainnat}
\bibliography{refs.bib}

\begin{thebibliography}{47}
\providecommand{\natexlab}[1]{#1}
\providecommand{\url}[1]{\texttt{#1}}
\expandafter\ifx\csname urlstyle\endcsname\relax
  \providecommand{\doi}[1]{doi: #1}\else
  \providecommand{\doi}{doi: \begingroup \urlstyle{rm}\Url}\fi

\bibitem[Arora et~al.(2019)Arora, Du, Hu, Li, and Wang]{arora2019fine}
Sanjeev Arora, Simon~S Du, Wei Hu, Zhiyuan Li, and Ruosong Wang.
\newblock Fine-grained analysis of optimization and generalization for
  overparameterized two-layer neural networks.
\newblock \emph{arXiv preprint arXiv:1901.08584}, 2019.

\bibitem[Basu(1955)]{basu1955}
D.~Basu.
\newblock On statistics independent of a complete sufficient statistic.
\newblock \emph{Sankhyā: The Indian Journal of Statistics (1933-1960)},
  15\penalty0 (4):\penalty0 377--380, 1955.
\newblock ISSN 00364452.
\newblock URL \url{http://www.jstor.org/stable/25048259}.

\bibitem[Belkin et~al.(2018)Belkin, Hsu, Ma, and Mandal]{belkin2018reconciling}
Mikhail Belkin, Daniel Hsu, Siyuan Ma, and Soumik Mandal.
\newblock Reconciling modern machine learning and the bias-variance trade-off.
\newblock \emph{arXiv preprint arXiv:1812.11118}, 2018.

\bibitem[Blei et~al.(2017)Blei, Kucukelbir, and McAuliffe]{blei2017variational}
David~M Blei, Alp Kucukelbir, and Jon~D McAuliffe.
\newblock Variational inference: A review for statisticians.
\newblock \emph{Journal of the American statistical Association}, 112\penalty0
  (518):\penalty0 859--877, 2017.

\bibitem[Blundell et~al.(2015)Blundell, Cornebise, Kavukcuoglu, and
  Wierstra]{blundell2015weight}
Charles Blundell, Julien Cornebise, Koray Kavukcuoglu, and Daan Wierstra.
\newblock Weight uncertainty in neural network.
\newblock In \emph{International Conference on Machine Learning}, pages
  1613--1622, 2015.

\bibitem[Damianou and Lawrence(2013)]{damianou13a}
Andreas Damianou and Neil Lawrence.
\newblock Deep gaussian processes.
\newblock volume~31 of \emph{Proceedings of Machine Learning Research}, pages
  207--215, Scottsdale, Arizona, USA, 29 Apr--01 May 2013. PMLR.
\newblock URL \url{http://proceedings.mlr.press/v31/damianou13a.html}.

\bibitem[de~G.~Matthews et~al.(2018)de~G.~Matthews, Hron, Rowland, Turner, and
  Ghahramani]{matthews2018gaussian}
Alexander~G. de~G.~Matthews, Jiri Hron, Mark Rowland, Richard~E. Turner, and
  Zoubin Ghahramani.
\newblock Gaussian process behaviour in wide deep neural networks.
\newblock In \emph{International Conference on Learning Representations}, 2018.
\newblock URL \url{https://openreview.net/forum?id=H1-nGgWC-}.

\bibitem[Dutordoir et~al.(2020)Dutordoir, van~der Wilk, Artemev, and
  Hensman]{dutordoir20a}
Vincent Dutordoir, Mark van~der Wilk, Artem Artemev, and James Hensman.
\newblock Bayesian image classification with deep convolutional gaussian
  processes.
\newblock volume 108 of \emph{Proceedings of Machine Learning Research}, pages
  1529--1539, Online, 26--28 Aug 2020. PMLR.
\newblock URL \url{http://proceedings.mlr.press/v108/dutordoir20a.html}.

\bibitem[Duvenaud et~al.(2016)Duvenaud, Maclaurin, and
  Adams]{duvenaud2016early}
David Duvenaud, Dougal Maclaurin, and Ryan Adams.
\newblock Early stopping as nonparametric variational inference.
\newblock In \emph{Artificial Intelligence and Statistics}, pages 1070--1077,
  2016.

\bibitem[Dziugaite and Roy(2017)]{dziugaite_computing_2017}
Gintare~Karolina Dziugaite and Daniel~M Roy.
\newblock Computing nonvacuous generalization bounds for deep (stochastic)
  neural networks with many more parameters than training data.
\newblock \emph{arXiv preprint arXiv:1703.11008}, 2017.

\bibitem[Dziugaite and Roy(2018)]{dziugaite_data-dependent_2018}
Gintare~Karolina Dziugaite and Daniel~M Roy.
\newblock Data-dependent {PAC}-{Bayes} priors via differential privacy.
\newblock In S.~Bengio, H.~Wallach, H.~Larochelle, K.~Grauman, N.~Cesa-Bianchi,
  and R.~Garnett, editors, \emph{{NeurIPS} 31}, pages 8430--8441. 2018.

\bibitem[Fort et~al.(2019)Fort, Nowak, Jastrzebski, and
  Narayanan]{fort2019stiffness}
Stanislav Fort, Pawe{\l}~Krzysztof Nowak, Stanislaw Jastrzebski, and Srini
  Narayanan.
\newblock Stiffness: A new perspective on generalization in neural networks.
\newblock \emph{arXiv preprint arXiv:1901.09491}, 2019.

\bibitem[Frankle and Carbin(2019)]{frankle2018the}
Jonathan Frankle and Michael Carbin.
\newblock The lottery ticket hypothesis: Finding sparse, trainable neural
  networks.
\newblock In \emph{International Conference on Learning Representations}, 2019.
\newblock URL \url{https://openreview.net/forum?id=rJl-b3RcF7}.

\bibitem[Gal and Ghahramani(2016)]{gal2016dropout}
Yarin Gal and Zoubin Ghahramani.
\newblock Dropout as a bayesian approximation: Representing model uncertainty
  in deep learning.
\newblock In \emph{international conference on machine learning}, pages
  1050--1059, 2016.

\bibitem[Germain et~al.(2016)Germain, Bach, Lacoste, and
  Lacoste-Julien]{germain_pac-bayesian_2016}
Pascal Germain, Francis Bach, Alexandre Lacoste, and Simon Lacoste-Julien.
\newblock {PAC}-{Bayesian} theory meets {Bayesian} inference.
\newblock In \emph{Advances in {Neural} {Information} {Processing} {Systems}},
  pages 1884--1892, 2016.

\bibitem[Graves(2011)]{graves2011practical}
Alex Graves.
\newblock Practical variational inference for neural networks.
\newblock In J.~Shawe-Taylor, R.~S. Zemel, P.~L. Bartlett, F.~Pereira, and
  K.~Q. Weinberger, editors, \emph{Advances in Neural Information Processing
  Systems 24}, pages 2348--2356. Curran Associates, Inc., 2011.
\newblock URL
  \url{http://papers.nips.cc/paper/4329-practical-variational-inference-for-neural-networks.pdf}.

\bibitem[Hardt et~al.(2015)Hardt, Recht, and Singer]{hardt2015train}
Moritz Hardt, Benjamin Recht, and Yoram Singer.
\newblock Train faster, generalize better: Stability of stochastic gradient
  descent, 2015.

\bibitem[He et~al.(2020)He, Lakshminarayanan, and Teh]{he2020bayesian}
Bobby He, Balaji Lakshminarayanan, and Yee~Whye Teh.
\newblock Bayesian deep ensembles via the neural tangent kernel.
\newblock \emph{arXiv preprint arXiv:2007.05864}, 2020.

\bibitem[Hinton and Van~Camp(1993)]{hinton1993keeping}
Geoffrey~E Hinton and Drew Van~Camp.
\newblock Keeping the neural networks simple by minimizing the description
  length of the weights.
\newblock In \emph{Proceedings of the sixth annual conference on Computational
  learning theory}, pages 5--13, 1993.

\bibitem[Jacot et~al.(2018)Jacot, Gabriel, and Hongler]{jacot2018neural}
Arthur Jacot, Franck Gabriel, and Cl{\'e}ment Hongler.
\newblock Neural tangent kernel: Convergence and generalization in neural
  networks.
\newblock In \emph{Advances in neural information processing systems}, pages
  8571--8580, 2018.

\bibitem[Jiang et~al.(2019)Jiang, Neyshabur, Mobahi, Krishnan, and
  Bengio]{jiang2019fantastic}
Yiding Jiang, Behnam Neyshabur, Hossein Mobahi, Dilip Krishnan, and Samy
  Bengio.
\newblock Fantastic generalization measures and where to find them, 2019.

\bibitem[Kalimeris et~al.(2019)Kalimeris, Kaplun, Nakkiran, Edelman, Yang,
  Barak, and Zhang]{kalimeris2019sgd}
Dimitris Kalimeris, Gal Kaplun, Preetum Nakkiran, Benjamin Edelman, Tristan
  Yang, Boaz Barak, and Haofeng Zhang.
\newblock Sgd on neural networks learns functions of increasing complexity.
\newblock In \emph{Advances in Neural Information Processing Systems}, pages
  3491--3501, 2019.

\bibitem[Khan et~al.(2019)Khan, Immer, Abedi, and Korzepa]{khan2019approximate}
Mohammad Emtiyaz~E Khan, Alexander Immer, Ehsan Abedi, and Maciej Korzepa.
\newblock Approximate inference turns deep networks into gaussian processes.
\newblock In H.~Wallach, H.~Larochelle, A.~Beygelzimer, F.~d'~Alch\'{e}-Buc,
  E.~Fox, and R.~Garnett, editors, \emph{Advances in Neural Information
  Processing Systems 32}, pages 3094--3104. Curran Associates, Inc., 2019.
\newblock URL
  \url{http://papers.nips.cc/paper/8573-approximate-inference-turns-deep-networks-into-gaussian-processes.pdf}.

\bibitem[Kullback and Leibler(1951)]{kullback1951}
S.~Kullback and R.~A. Leibler.
\newblock On information and sufficiency.
\newblock \emph{Ann. Math. Statist.}, 22\penalty0 (1):\penalty0 79--86, 03
  1951.
\newblock \doi{10.1214/aoms/1177729694}.
\newblock URL \url{https://doi.org/10.1214/aoms/1177729694}.

\bibitem[Lakshminarayanan et~al.(2017)Lakshminarayanan, Pritzel, and
  Blundell]{lakshminarayanan2017simple}
Balaji Lakshminarayanan, Alexander Pritzel, and Charles Blundell.
\newblock Simple and scalable predictive uncertainty estimation using deep
  ensembles.
\newblock In \emph{Advances in neural information processing systems}, pages
  6402--6413, 2017.

\bibitem[Lee et~al.(2018)Lee, Sohl-dickstein, Pennington, Novak, Schoenholz,
  and Bahri]{lee2018deep}
Jaehoon Lee, Jascha Sohl-dickstein, Jeffrey Pennington, Roman Novak, Sam
  Schoenholz, and Yasaman Bahri.
\newblock Deep neural networks as gaussian processes.
\newblock In \emph{International Conference on Learning Representations}, 2018.
\newblock URL \url{https://openreview.net/forum?id=B1EA-M-0Z}.

\bibitem[MacKay(1992)]{mackay1992bayesian}
David~JC MacKay.
\newblock \emph{Bayesian methods for adaptive models}.
\newblock PhD thesis, California Institute of Technology, 1992.

\bibitem[MacKay(2003)]{mackay2003information}
David~JC MacKay.
\newblock \emph{Information theory, inference and learning algorithms}.
\newblock Cambridge university press, 2003.

\bibitem[Maddox et~al.(2019)Maddox, Izmailov, Garipov, Vetrov, and
  Wilson]{maddox2019simple}
Wesley~J Maddox, Pavel Izmailov, Timur Garipov, Dmitry~P Vetrov, and
  Andrew~Gordon Wilson.
\newblock A simple baseline for bayesian uncertainty in deep learning.
\newblock In \emph{Advances in Neural Information Processing Systems}, pages
  13132--13143, 2019.

\bibitem[Mandt et~al.(2017)Mandt, Hoffman, and Blei]{mandt2017stochastic}
Stephan Mandt, Matthew~D Hoffman, and David~M Blei.
\newblock Stochastic gradient descent as approximate bayesian inference.
\newblock \emph{The Journal of Machine Learning Research}, 18\penalty0
  (1):\penalty0 4873--4907, 2017.

\bibitem[Matthews et~al.(2017)Matthews, Hron, Turner, and
  Ghahramani]{matthews2017}
Alexander G de~G Matthews, Jiri Hron, Richard~E Turner, and Zoubin Ghahramani.
\newblock Sample-then-optimize posterior sampling for bayesian linear models.
\newblock \emph{Neural Information Processing Systems}, 2017.

\bibitem[McAllester(1999)]{mcallester1999}
David~A. McAllester.
\newblock Some {PAC}-{Bayesian} {Theorems}.
\newblock \emph{Machine Learning}, 37\penalty0 (3):\penalty0 355--363, 1999.

\bibitem[Nagarajan and Kolter(2019)]{nagarajan2019uniform}
Vaishnavh Nagarajan and J.~Zico Kolter.
\newblock Uniform convergence may be unable to explain generalization in deep
  learning.
\newblock In H.~Wallach, H.~Larochelle, A.~Beygelzimer, F.~d'~Alche-Buc,
  E.~Fox, and R.~Garnett, editors, \emph{Advances in Neural Information
  Processing Systems 32}, pages 11615--11626. Curran Associates, Inc., 2019.

\bibitem[Nakkiran et~al.(2019)Nakkiran, Kaplun, Bansal, Yang, Barak, and
  Sutskever]{nakkiran2019deep}
Preetum Nakkiran, Gal Kaplun, Yamini Bansal, Tristan Yang, Boaz Barak, and Ilya
  Sutskever.
\newblock Deep double descent: Where bigger models and more data hurt.
\newblock \emph{arXiv preprint arXiv:1912.02292}, 2019.

\bibitem[Negrea et~al.(2019)Negrea, Haghifam, Dziugaite, Khisti, and
  Roy]{negrea2019information}
Jeffrey Negrea, Mahdi Haghifam, Gintare~Karolina Dziugaite, Ashish Khisti, and
  Daniel~M Roy.
\newblock Information-theoretic generalization bounds for sgld via
  data-dependent estimates.
\newblock In \emph{Advances in Neural Information Processing Systems}, pages
  11015--11025, 2019.

\bibitem[Osband et~al.(2018)Osband, Aslanides, and
  Cassirer]{osband2018randomized}
Ian Osband, John Aslanides, and Albin Cassirer.
\newblock Randomized prior functions for deep reinforcement learning.
\newblock In \emph{Advances in Neural Information Processing Systems}, pages
  8617--8629, 2018.

\bibitem[Rahimi and Recht(2008)]{rahimi2008random}
Ali Rahimi and Benjamin Recht.
\newblock Random features for large-scale kernel machines.
\newblock In \emph{Advances in neural information processing systems}, pages
  1177--1184, 2008.

\bibitem[Rasmussen(2003)]{rasmussen2003gaussian}
Carl~Edward Rasmussen.
\newblock Gaussian processes in machine learning.
\newblock In \emph{Summer School on Machine Learning}, pages 63--71. Springer,
  2003.

\bibitem[Rasmussen and Ghahramani(2001)]{rasmussen2001occam}
Carl~Edward Rasmussen and Zoubin Ghahramani.
\newblock Occam's razor.
\newblock In \emph{Advances in neural information processing systems}, pages
  294--300, 2001.

\bibitem[Ru et~al.(2020)Ru, Lyle, Schut, van~der Wilk, and
  Gal]{ru2020revisiting}
Binxin Ru, Clare Lyle, Lisa Schut, Mark van~der Wilk, and Yarin Gal.
\newblock Revisiting the train loss: an efficient performance estimator for
  neural architecture search, 2020.

\bibitem[Smith and Le(2017)]{smith2017bayesian}
Samuel~L Smith and Quoc~V Le.
\newblock A bayesian perspective on generalization and stochastic gradient
  descent.
\newblock \emph{arXiv preprint arXiv:1710.06451}, 2017.

\bibitem[Smith and Le(2018)]{smith2018}
Samuel~L. Smith and Quoc~V. Le.
\newblock A bayesian perspective on generalization and stochastic gradient
  descent.
\newblock In \emph{International Conference on Learning Representations}, 2018.
\newblock URL \url{https://openreview.net/forum?id=BJij4yg0Z}.

\bibitem[Valle-P{\'e}rez et~al.(2018)Valle-P{\'e}rez, Camargo, and
  Louis]{valle2018deep}
Guillermo Valle-P{\'e}rez, Chico~Q Camargo, and Ard~A Louis.
\newblock Deep learning generalizes because the parameter-function map is
  biased towards simple functions.
\newblock \emph{arXiv preprint arXiv:1805.08522}, 2018.

\bibitem[van~der Wilk et~al.(2018)van~der Wilk, Bauer, John, and
  Hensman]{van_der_wilk_learning_2018}
M.~van~der Wilk, M.~Bauer, S.~John, and J.~Hensman.
\newblock Learning {Invariances} using the {Marginal} {Likelihood}.
\newblock \emph{arXiv e-prints}, August 2018.
\newblock \_eprint: 1808.05563.

\bibitem[Welling and Teh(2011)]{welling2011bayesian}
Max Welling and Yee~W Teh.
\newblock Bayesian learning via stochastic gradient langevin dynamics.
\newblock In \emph{Proceedings of the 28th international conference on machine
  learning (ICML-11)}, pages 681--688, 2011.

\bibitem[Wilson and Izmailov(2020)]{wilson2020bayesian}
Andrew~Gordon Wilson and Pavel Izmailov.
\newblock Bayesian deep learning and a probabilistic perspective of
  generalization.
\newblock \emph{arXiv preprint arXiv:2002.08791}, 2020.

\bibitem[Zhang et~al.(2016)Zhang, Bengio, Hardt, Recht, and
  Vinyals]{zhang2016understanding}
Chiyuan Zhang, Samy Bengio, Moritz Hardt, Benjamin Recht, and Oriol Vinyals.
\newblock Understanding deep learning requires rethinking generalization.
\newblock \emph{arXiv preprint arXiv:1611.03530}, 2016.

\end{thebibliography}
\newpage 
\appendix

\section{Proofs of Theoretical Results} \label{sec:proofs}

\PropLk*
\begin{proof}
The result for $\mathcal{L}$ follows from a straightforward derivation:
\begin{align}
    \mathcal{L}(\data) &= \sum \int \log P(\data_i|\theta) dP(\theta|\data_{<i})  \\
    &= \sum \int \log [\frac{ P(\data_i|\theta) P(\theta | \data_{<i}) P(\data_i|\data_{<i})}{P(\theta | \data_{<i}) P(\data_i|\data_{<i})}] dP(\theta|\data_{<i})\\
    &= \sum \int \log\frac{ P(\theta|\data_{\leq i}))}{P(\theta|\data_{<i})} dP(\theta|\data_{<i}) + \sum \log P(\data_i| \data_{<i}) \\
    &= \sum  \bigg (  \log P(\data_i|\data_{<i})-\KL(P(\theta|\data_{<i})|| P(\theta|\data_{\leq i})) \bigg ) \\
    &= \log P(\data) - \sum_{i=1}^n \KL(P(\theta|\data_{<i})||P(\theta|\data_{\leq i})).
\end{align}
The result for $\hat{\mathcal{L}}_k$ follows immediately from Jensen's inequality, yielding
\begin{equation}
    \sum \mathbb{E}[\log \sum_{j=1}^k \frac{1}{k} p(\data_i|\theta_j)] \leq \sum \log \mathbb{E}[ \sum_{j=1}^k \frac{1}{k} p(\data_i|\theta_j)] =\sum \log \mathbb{E}[ p(\data_i|\theta_j)] =  \log P(\data) \; .
\end{equation}
Because $\mathcal{L}_k$ applies Jensen's inequality to a random variable with decreasing variance as a function of $k$, we expect the bias of $\mathcal{L}_k$ to decrease as $k$ grows, an observation characterized in Section \ref{sec:experiments}.
\end{proof}
\PropLS*
\begin{proof}
To show that the sum of the estimated log likelihoods is a lower bound on the log marginal likelihood, it suffices to show that each term in the sum of the estimates is a lower bound on the corresponding term in log marginal likelihood expression. Thus, without loss of generality we consider a single data point $\data_i = (x, y)$ and posterior distribution $p(y|x, \data_{<i})=\mathcal{N}(\mu, \sigma^2)$. 

Let $y \in \mathbb{R}$, $\hat{\mu}, \hat{\sigma}$ the standard estimators for sample mean and variance given sample $\hat{Y} \in \mathbb{R}^k$ sampled from $\mathcal{N}(\mu, \sigma^2)$. We want to show

\begin{equation}
\mathbb{E}_{\hat{Y} \sim \mathcal{N}(\mu, \sigma^2)}[\ln p(y|\hat{\mu}, \hat{\sigma}^2)] \leq \ln p(y|\mu, \sigma^2).
\end{equation}
We first note that $\hat{\mu}(\hat{Y}) \perp \hat{\sigma}(\hat{Y})$ for $\hat{Y}$ a collection of i.i.d. Gaussian random variables \citep{basu1955}. 
We also take advantage of the fact that the log likelihood of a Gaussian is concave with respect to its $\mu$ parameter and its $\sigma^2$ parameter. Notably, the log likelihood is \textit{not} concave w.r.t. the joint pair $(\mu, \sigma^2)$, but because the our estimators are independent, this will not be a problem for us. 
We proceed as follows by first decomposing the expectation over the samples $\hat{Y}$ into an expectation over $\hat{\mu}$ and $\widehat{\sigma^2}$
\begin{align}
\mathbb{E}_{X \sim \mathcal{N}(\mu, \sigma^2)}[\ln p(y| \hat{\mu}, \hat{\sigma}^2)] &= \mathbb{E}_{\hat{\mu}, Y_2, \dots, Y_N} \ln p(y|\hat{\mu}, \hat{\sigma}^2) \\
&= \mathbb{E}_{\hat{\mu}} \mathbb{E}_{\hat{\sigma}^2} \ln p(y|\hat{\mu}, \hat{\sigma}^2)
\intertext{We apply Jensen's inequality first to the inner expectation, then to the outer.}
&\leq \mathbb{E}_{\hat{\mu}} \ln p(y|\hat{\mu}, \mathbb{E}[\hat{\sigma}^2]) =  \mathbb{E}_{\hat{\mu} }\ln p(y|\hat{\mu}, \sigma^2)   \\
&\leq \ln p(y|\mu, \sigma^2)
\end{align}
So we obtain our lower bound.
\end{proof}

\ThmSTO*

\begin{proof}
The heavy lifting for this result has largely been achieved by Propositions \ref{prop:lk} and \ref{prop:ls}, which state that provided the samples $\theta^{i}_j$ are distributed according to the posterior, the inequalities will hold. It therefore remains only to show that the sample-then-optimize procedure yields samples from the posterior. The proof of this result can be found in Lemma 3.8 of \citet{osband2018randomized}, who show that the optimum for the gradient descent procedure described in Algorithm \ref{alg:estimate} does indeed correspond to the posterior distribution for each subset $\data_{<i}$. 

Finally, it is straightforward to express the lower bound estimator $\hat{\mathcal{L}}$ as the sum of regression losses. We obtain this result by showing that the inequality holds for each term $\log P(\data_i|\theta_i)$ in the summation. 

\begin{align}
 \log P(\data_i|\theta) &= \log[ \exp \bigg (-\frac{(\theta^\top x_i - y_i)^2 }{2 \sigma^2} \bigg )\frac{1}{\sqrt{2\pi}\sigma} ] \\
 &= -\frac{(\theta^\top x_i - y_i)^2 }{2 \sigma^2} -\frac{1}{2} \log (2 \pi \sigma^2) \\
 &= c_1 \ell_2(\data_i, \theta) + c_2
\end{align}

We note that in practice, the solutions found by gradient descent for finite step size and finite number of steps will not necessarily correspond to the exact local optimum. However, it is straightforward to bound the error obtained from this approximate sampling in terms of the distance of $\theta$ from the optimum $\theta^*$. Denoting the difference $|\theta - \theta^*|$ by $\delta$, we get
\begin{align}
   | \log P(\data_i|\theta^*) - \log P(\data_i|\theta)| &=  |
    \frac{((\theta^*)^\top x_i - y_i)^2 }{2 \sigma^2} -  \frac{((\theta)^\top x_i - y_i)^2 }{2 \sigma^2}| \\
    & \leq 
    \frac{1}{2 \sigma^2} | (\theta^*)^\top x_i - \theta^\top x_i|^2 \\
    &\leq |( (\theta^*)^\top x_i)^2 - (\theta^\top x_i)^2| + |2y||\theta^\top x - (\theta^*)^\top x| \\
    &\leq |(\theta^* - \theta)^\top x + 2((\theta^*)^\top x)((\theta^*-\theta)^\top x)| + |2y||\theta^\top x - (\theta^*)^\top x| \\
    &\leq |\theta^* - \theta||x| + 2|\theta^* x||\theta^* - \theta||x| + |2y||x||\theta - \theta^*|
\end{align}
and so the error in the estimate of $\log P(\data | \theta)$ will be proportional to the distance $|\theta - \theta^*|$ induced by the approximate optimization procedure.
\end{proof}

\CorNTK*
\begin{proof}
Follows immediately from the results of \citet{he2020bayesian} stating that the the limiting distribution of $f^k_\infty$ is precisely $P(f|\data^n_{\le k}, \model)$. We therefore obtain the same result as for Theorem \ref{thm:sto}, plugging in the kernel gradient descent procedure on $f$ for the parameter-space gradient descent procedure on $\theta$.
\end{proof}
The following Lemma will be useful in order to prove Proposition~\ref{prop:modelselect}. Intuitively, this result states that in a linear regression problem in which each feature $\phi_i$ is `normalized' (the dot product $\langle \phi_i, y \rangle = \langle \phi_j, y \rangle = \alpha$ for some $\alpha$ and all $i, j$) and `independent' (i.e. $\langle \Pi_{y^\perp} \phi_i, \Pi_{y^\perp} \phi_j \rangle = 0$), then the optimal linear regression solution assigns highest weight to the feature which obtains the least error in predicting $y$ on its own.
\begin{lemma}
Let $y \in \mathbb{R}^n$, and $\Phi \in \mathbb{R}^{d \times d}$ be a design matrix such that $\Phi[:, j] = \alpha y + \epsilon_j \forall j$ for some fixed $\alpha \geq 0$, with $\epsilon \in y^\perp$, and $\epsilon_i^\top \epsilon_j = 0$ for all $i \neq j$. Let $w^*$ be the solution to the least squares regression problem on $\Phi$ and $y$. Then 
\begin{equation}
    \min_i w_i = \min_i \|f_i(x) - y\|^2 = \max_i \mathcal{L}(\model_i)
\end{equation}
\end{lemma}
\begin{proof}
We express the minimization problem as follows. We let $\phi(x)$ = $( f_1(x), \dots, f_k(x))$, where $f_i(x) = \alpha y + \epsilon_i$, with $\epsilon_i \perp \epsilon_j $. We denote by $\mathbbm{1}$ the vector containing all ones (of length $k$). We observe that we can decompose the design matrix $\Phi$ into one component whose columns are parallel to $y$, denoted $\Phi_y$, and one component whose columns are orthogonal to $y$, denoted $\Phi_\perp$. Let $\sigma^2_i = \|\epsilon_i\|^2$.  By assumption, $\Phi_y = \alpha y \mathbbm{1}^\top$, and $\Phi_\perp^\top \Phi_\perp = \text{diag}(\sigma^2_1, \dots, \sigma^2_n) = \Sigma$. We then observe the following decomposition of the squared error loss of a weight vector $w$, denoted $\ell(w)$.
\begin{align*}
\ell(w) &= \| \Phi w - y\|^2 = (\Phi w - y)^\top (\Phi w - y) \\
&= ((\Phi_y + \Phi_\perp) w - y)^\top ((\Phi_y + \Phi_\perp)w - y)\\
&=(\Phi_y w - y)^\top (\Phi_y w - y) + w^\top \Phi_\perp^\top \Phi_\perp w \\
&= \|y\|^2 \|1 - \alpha \mathbbm{1}^\top w \|^2  + \sum \sigma_i^2 w_i \\
\end{align*}
In particular, the loss decomposes into a term which depends on the sum of the $w_i$, and another term which will depend on the norm of the component of each model's predictions orthogonal to the targets $y$.


As this is a quadratic optimization problem, it is clear that an optimal $w$ exists, and so $w^\top \mathbbm{1}$ will take some finite value, say $\beta$. We will show that for any fixed $\beta$, the solution to the minimization problem
\begin{equation}
    \min_w \sum w_i \sigma_i^2 : w^\top \mathbbm{1} = \beta
\end{equation}
is such that the argmax over $i$ of $w_i$ is equal to that of the minimum variance. This follows by applying the method of Lagrange multipliers to obtain that the optimal $w$ satisfies \begin{equation}w^*_i = \frac{\alpha}{\sum \sigma_i^{-2}} \frac{1}{\sigma_i^2}.\end{equation}
In particular, $w^*_i$ is inversely proportional to the variance of $f_i$, and so is maximized for $i = \argmin_i \mathbb{E}\|f_i(x) - y\|^2$. 

\end{proof}
\PropBMS*
\begin{proof}
We first clarify the independence assumptions as they pertain to the assumptions of the previous lemma: writing $\Phi[:, i]$ as $f_i(x) + \zeta_i = \alpha y + \epsilon_i + \zeta_i$ with $\zeta_i \sim \mathcal{N}(0, \Sigma_i)$ corresponding to the noise from the posterior distribution and $f_i$ its mean, the first independence assumption is equivalent to the requirement that $f_i = \alpha' y + \epsilon_i$ with $\epsilon_i \perp y$ for all $i$. The second independence assumption is an intuitive expression of the constraint that $\epsilon_i \perp \epsilon_j$ in the linear-algebraic sense of independence, and that $\zeta_i^j$ is sampled independently (in the probabilistic sense) for all $i$ and $j$. 

We note that our lower bound for each model in the linear regression setting is equal to $\mathbb{E} \sum_{i=1}^N \|f_k(x_i) + \zeta_i - y_i\|^2 + c$ where $c$ is a fixed normalizing constant. By the previous Lemma, we know that the linear regression solution $w^*$ based on the posterior means satisfies, $\max_i w^*_i = \max_i \mathcal{L}(\model_i)$. It is then straightforward to extend this result to the noisy setting.
\begin{align}
    \mathbb{E}[ \|\Phi w - y\|^2] &= \mathbb{E}[\|(\Phi_y + \Phi_\perp + \zeta)w - y\|^2] \\
    &= \mathbb{E}[((\Phi_y + \Phi_\perp + \zeta)w - y)^\top ((\Phi_y + \Phi_\perp + \zeta)w - y)] \\
    &= \|\Phi_y w - y\|^2 + w^\top \Phi_\perp ^\top  \Phi_\perp w + \mathbb{E}[w^\top \zeta^\top \zeta w] \\
    &=  (w^\top \mathbbm{1} - \alpha)^2\|y\|^2 + w^\top \Phi_\perp ^\top  \Phi_\perp w + \mathbb{E}[w^\top \zeta^\top \zeta w] \\
    &= (w^\top \mathbbm{1} - \alpha)^2\|y\|^2  + \sum w_i^2( \|\Phi_\perp[:, i] \|^2 + \|\zeta_i\|^2)
\end{align}
We again note via the same reasoning as in the previous Lemma that the model with the greatest lower bound will be the one which minimizes $\|\Phi_\perp[:, i]\|^2 + \|\zeta_i\|^2$, and that the weight given to index $i$ will be inversely proportional to this term.

It only remains to show that for each model $i$, the model which maximizes $\mathcal{L}(M_i)$ will also minimize $\|\Phi_\perp[:, i]\|^2 + \|\zeta_i\|^2$. This follows precisely from the Gaussian likelihood assumption. As we showed previously 
\begin{align}
    \mathcal{L}(\data | \model_i) = \mathbb{E}[\sum \log P(y_i | \data_{<i})] &\propto - 
\sum \mathbb{E}[\ell_2(y_i - \hat{y}_i] \\
    &= [\| y - \mu\|^2 + \mathbb{E}[\|\hat{y} - \mu \|^2] \\
    &= \alpha\|y\|^2 + \|\Phi_\perp[:, i]\|^2 + \mathbb{E}[\|\zeta_i\|^2]
\end{align}
and so finding the model $\model_i$ which maximizes $\mathcal{L}(\data, \model_i)$ is equivalent to picking the maximal index $i$ of $w^*$ which optimizes the expected loss of the least squares regression problem.
\end{proof}
\clearpage

\section{Experiments}

\subsection{Experimental details: Model Selection using Trajectory Statistics } \label{sec:ex_ms_blr_synthetic_data}

\begin{figure}
    \centering
    \includegraphics[width=0.325\linewidth]{figures/model_selection_linear/feature_selection.png}
    \includegraphics[width=0.325\linewidth]{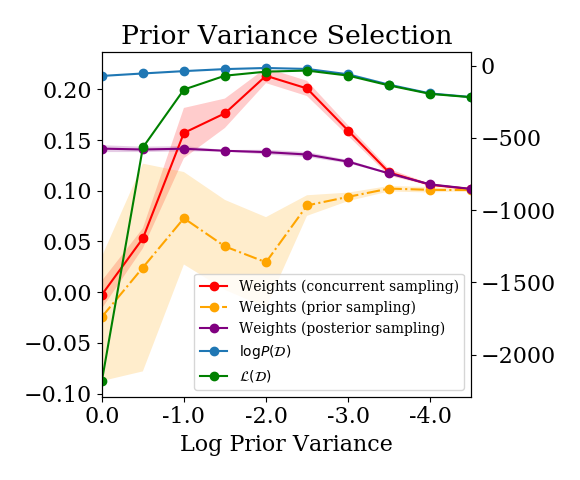}
    \includegraphics[width=0.325\linewidth]{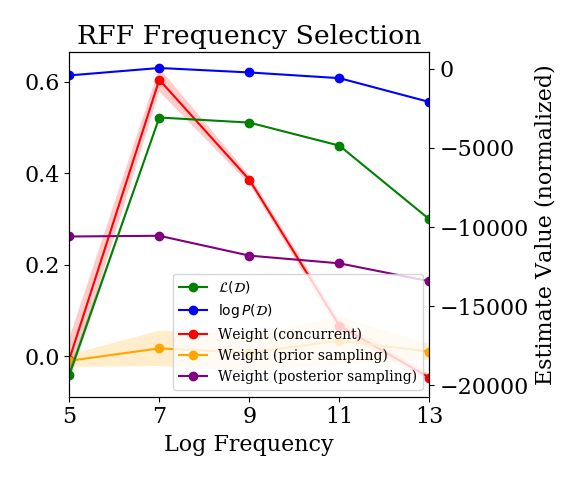}
    \hfill 
    \caption{Relative rankings given by optimize-then-prune, ML, and estimated $\mathcal{L}(\data)$. Left: feature selection. Middle: prior variance selection. Right: RFF frequency selection. Rankings are consistent with what our theoretical results predict. Results are averaged over $5$ runs.
    }
    \label{fig:app_ml_v_weight}
\end{figure}

We consider 3 model selection settings in which to evaluate the practical performance of our estimators. In \textbf{prior variance selection} we evaluate a set of BLR models on a synthetic linear regression data set. Each model $\mathcal{M}_i$ has a prior distribution over the $d$ parameters of the form $w \sim \mathcal{N}(0, \sigma_i^2 I_d)$ for some $\sigma_i^2$, and the goal is to select the optimal prior variance (in other words, the optimal regularization coefficient). We additionally evaluate an analogous initialization variance selection method on an NTK network trained on a toy regression dataset. In \textbf{frequency (lengthscale) selection} we use as input a subset of the handwritten digits dataset MNIST given by all inputs labeled with a 0 or a 1. We compute random Fourier features (RFF) of the input to obtain the features for a Bayesian linear regression model, and perform model selection over the frequency of the features (full details on this in the appendix). This is equivalent to obtaining the lengthscale of an approximate radial basis function kernel. In \textbf{feature dimension selection}, we use a synthetic dataset \citep{wilson2020bayesian} of the form $(\textbf{X}, \textbf{y})$, where $x_i = (y_i + \epsilon_1,  y_i + \dots, y_i + \epsilon_{15}, \epsilon_{16}, \dots, \epsilon_{30})$. We then consider a set of models $\{\model_k\}$ with feature embeddings $\phi_k(x_i) = x_i[1, \dots, k]$. The optimal model in this setting is the one which uses exactly the set of `informative' features $x[1, \dots, 15]$. 

The synthetic data simulation used in this experiment is identical to that used in \citep{wilson2020bayesian}. Below, we provide the details. 

Let $k$ be the number of informative features and $d$ the total number of features. We generate a datapoint $\data_i  = \{x_i,y_i\}$ as follows:
\begin{enumerate}
    \item {Sample $y_i$}: $y_i \sim U([0, 1])$
    \item {Sample $k$ informative features}: $x_{i,j} \sim N(y_i, \sigma_0) \quad \forall j \in 1, \dots k$
    \item {Sample $\max(d-k,0)$ noise features}: $x_{i,k+j} \sim N(0, \sigma_1) \quad \forall j \in 1, \dots d-k$
    \item {Concatenate the features}: $X_i= [x_{i,1}, \dots x_{i,d}]$
\end{enumerate}

We set $\sigma_0= \sigma_1=1$, $k = 15$, $n = 30$, and let $d$ vary from $5$ to $n$. We then run our estimators on the Bayesian linear regression problem for each feature dimension, and find that all estimators agree on the optimal number of features, $k$.

To compute the random fourier features used for MNIST classification, we vectorize the MNIST input images and follow the procedure outlined by \citep{rahimi2008random} (Algorithm 1) to produce RFF features, which are then used for standard Bayesian linear regression against the binarized labels. The frequency parameter (which can also be interpreted as a transformation of the lengthscale of the RBF kernel approximated by the RFF model) is the parameter of interest for model selection. Results can be found in Figure~\ref{fig:app_ml_v_weight}.

We additionally provide an analogue to our evaluation of model selection in NTK-GPs, with the change in the log marginal likelihood plotted instead of $\mathcal{L}(\data)$. We obtain analogous results, as can be seen in Figure~\ref{fig:app_ntk}.
\begin{figure}
\centering
\includegraphics[width=0.6\textwidth]{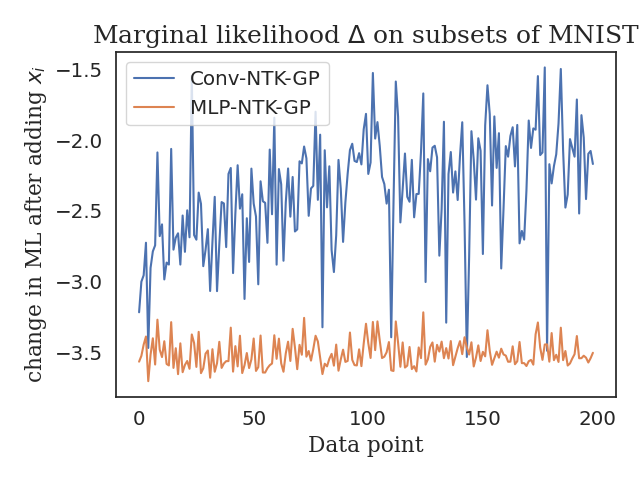}
\caption{Evaluation of change in log ML after data point $i$ is added for NTK-GPs on a random subset of MNIST.}
\label{fig:app_ntk}
\end{figure}

\subsection{Experimental details: Bayesian model comparison} \label{sec:exp_details_sgd_dnn}

Here we provide further detail of the experiment in Section 4.2.1.
The goal of the experiment is to determine whether the connection between sum-over-training losses (SOTL) and model evidence observed in the linear regression setting extends to DNNs. In particular, the two sub-questions are: 
\begin{enumerate}
    \item Do models with a lower SOTL generalize better?
    \item Are these models favoured by SGD? 
\end{enumerate}

To answer these questions, we train a linear combination of NNs. We can answer subquestion [1] by plotting the correlation between SOTL and test performance of an individual model. Further, we address subquestion [2] by considering the correlation between test loss and linear weights assigned to each model.

Below we explain the set-up of the linear combination in more detail. We train a variety of deep neural networks along with a linear `ensemble' layer that performs a linear transformation of the concatenated logit outputs\footnote{These are pre-softmax outputs. To obtain the predicted probability of a class, they are fed through a softmax function.} of the classification models. Let $h_m(x_i)$ be logit output of model $m$ for input $x_i$, $\ell(y_i, h_i)$ be the loss for point $i$ (where $h_i$ is a logit) and $w_{m,t}$ be the weight corresponding to model $m$ at time step $t$. 

We consider two training strategies: we first train models individually using the cross-entropy loss between each model's prediction and the true label, only cross-entropy loss of the final ensemble prediction to train the linear weights. Mathematically, we update the models using the gradients
\begin{equation}
    \frac{\delta}{\delta \theta_m} \ell(y_i, h_m(x_i)),
\end{equation}
and the `ensemble' weights using
\begin{equation}
    \frac{\delta}{\delta w_m} \ell( y_i, \sum_m w_m h_m(x_i)).
\end{equation}
We refer to this training scheme as \textit{Parallel Training} as the models are trained in parallel.  We also consider the setting in which the models are trained using the cross entropy loss from the ensemble prediction backpropagated through the linear ensemble layer, i.e. the model parameters are now updated using: 
\begin{equation}
    \frac{\delta}{\delta \theta_m} \ell(y_i, \sum_m w_m h_m(x_i)).
\end{equation}
We refer to this scheme as the \textit{Concurrent Training}. 

We train a variety of different MLPs (with varying layers,and nodes) and convolutional neural networks (with varying layers, nodes and kernels) on FashionMNIST using SGD until convergence.

\subsection{Experimental Details: SGD upweights submodels that perform well} \label{sec:exp_details_sgd_submodels}
Below we provide further details of the experiment in Section 4.2.2. The goal of the experiment is to determine whether SGD upweights sub-models that fit the data better.  

We train a MLP network (with units $200, 200, 10$) on FashionMMIST using SGD until convergence.  After training is completed, for every class of $y$, we rank all nodes in the penultimate layer by the norm of their absolute weight (in the final dense layer).  We group the points into submodels according to their ranking --  the $k$ nodes with the highest weights are grouped together, next the $k+1, \dots 2k$ ranked nodes are grouped, etc. We set $k=10$. 

We determine the performance of a submodels by training a simple logistic classifier to predict the class of an input, based on the output of the submodel. To measure the performance of the classifier, we use the cross-entropy loss. To capture the equivalent notion of the AUC, we estimate the performance of the sub-models throughout training, and sum over the estimated cross-entropy losses. 

Below, we show additional plots for the \textit{parallel} and \textit{concurrent} training schemes. The results are the same to those presented in the main text, and we observe  [1] a negative correlation between test performance and ensemble weights and [2] a strong correlation between SOTL and average test cross-entropy.

\begin{figure}[H]
    \begin{minipage}{.27\textwidth}
    \includegraphics[ width=\linewidth]{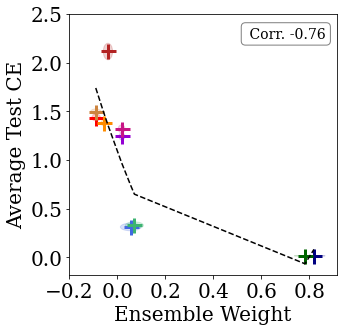}
    \end{minipage}
    \begin{minipage}{.27\textwidth}
    \includegraphics[ width=\linewidth]{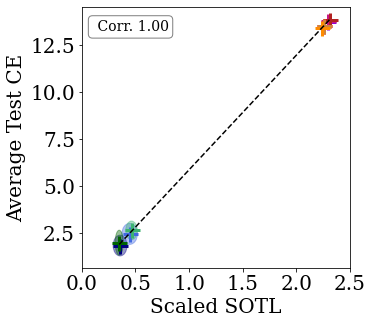}
    \end{minipage}
    \begin{minipage}{.27\textwidth}
    \includegraphics[ width=\linewidth]{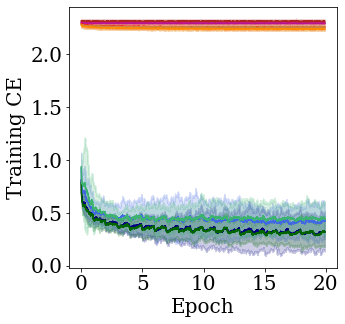}
    \end{minipage}
    \begin{minipage}{.17\textwidth}
    \includegraphics[ width=\linewidth]{figures/model_selection_dnn/legend.png}
    \end{minipage}
    \caption{\textbf{Linear combinations of DNNs on FashionMNIST.}  Left: ensemble weights versus the test loss for parallel training; we observe a negative correlation. Middle: SOTL (standardized by the number of training samples) versus test loss for concurrent and concurrent training. We observe a strong correlation indicating that the SOTL generalizes well. Right: training curves for the different models in concurrent training schemes. All results are averaged over $10$ runs, and standard deviations are shown by the shaded regions around each observation. The model parameters, given in the parentheses, are the number of layers ($l$), nodes per layer ($n$) and kernel size ($k$), respectively. }
    \label{fig:mod_select_dnn_parallel}
\end{figure}

However, similarly to the linear setting, the difference in assigned weights is magnified in the concurrent training scheme. Here we find that in the concurrent training scheme, the ensemble focuses on training the CNNs (as can be seen from the training curve in Figure \ref{fig:mod_select_dnn} in the main text). This is likely because CNNs are able to learn more easily, leading to larger weights earlier on.

\newpage 
Above, we show additional plots to those shown in Figure \ref{fig:sgd_submodel}, Section \ref{sec:sgd_submodel}. Figure \ref{fig:sgd_submodel_full} shows the results for the all FashionMNIST classes, and Figure \ref{fig:sgd_submodel_full_cifar} shows the results for experiment on CIFAR-10.  From both, we see that SGD assigns higher weights to subnetworks that perform better. 

\begin{figure}
    \centering
    \includegraphics[ width=\linewidth]{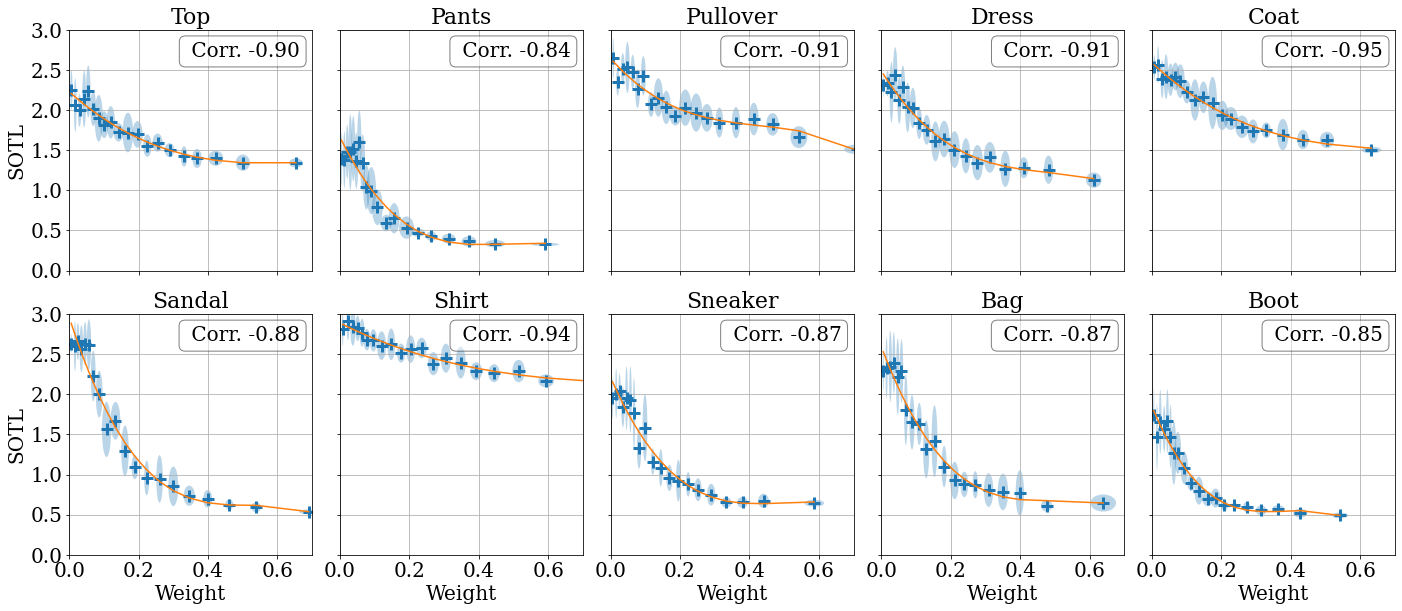}
    \caption{Weight assigned to subnetwork by SGD in a deep neural network (x-axis) versus the subnetwork performance (estimated by the sum of cross-entropy, on the y-axis) for different FashionMNIST classes. The light blue ovals denote depict $95\%$ confidence intervals, estimated over 10 seeds (i.e. 2$\sigma$ for both the weight and SOTL).  The orange line depicts the general trend.}
    \label{fig:sgd_submodel_full}
\end{figure}

\begin{figure}
    \centering
    \includegraphics[ width=\linewidth]{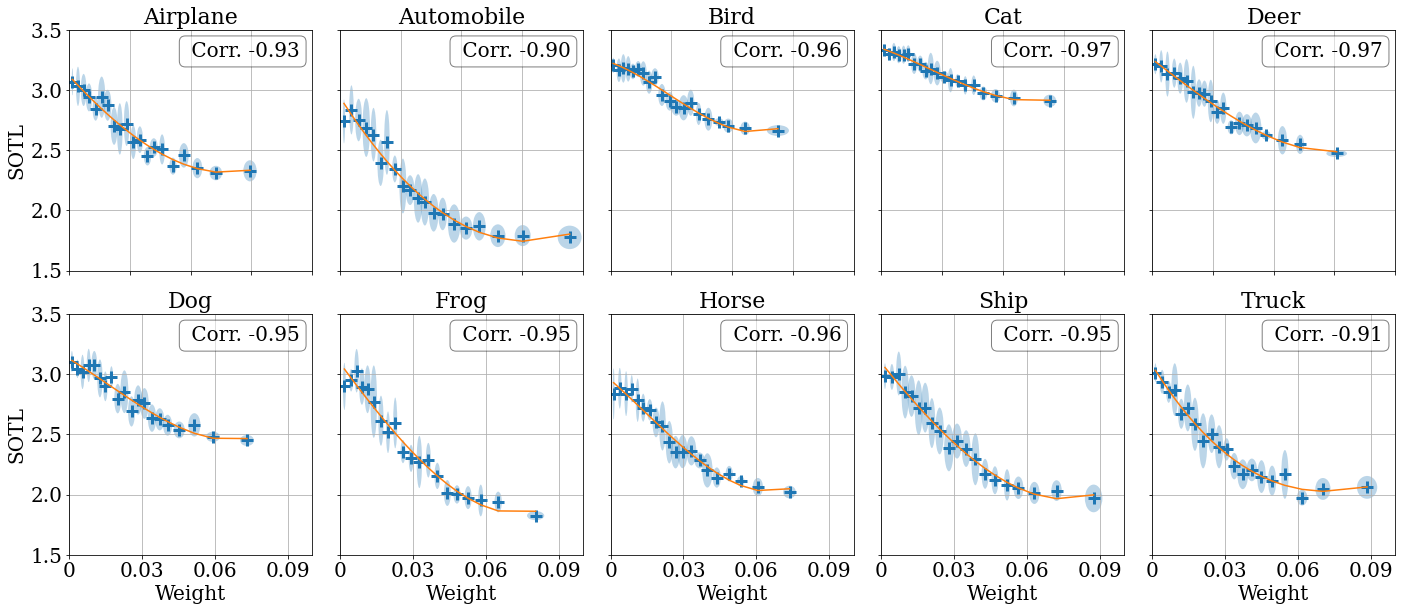}
    \caption{Weight assigned to subnetwork by SGD in a deep neural network (x-axis) versus the subnetwork performance (estimated by the sum of cross-entropy, on the y-axis) for different CIFAR-10 classes. The light blue ovals denote depict $95\%$ confidence intervals, estimated over 10 seeds (i.e. 2$\sigma$ for both the weight and SOTL).  The orange line depicts the general trend.}
    \label{fig:sgd_submodel_full_cifar}
\end{figure}

\vfill 
\newpage 

\end{document}